\documentclass[lettersize,journal]{IEEEtran}
\usepackage{amsmath,amssymb,amsfonts,amsthm}
\usepackage{algorithmic}
\usepackage{algorithm}
\usepackage{array}
\usepackage{booktabs}
\usepackage[caption=false,font=footnotesize,labelfont=sf,textfont=sf]{subfig}
\usepackage{textcomp}
\usepackage{stfloats}
\usepackage{url}
\usepackage{verbatim}
\usepackage{graphicx}
\usepackage{cite}

\usepackage{xcolor}
\usepackage{kotex}

\usepackage{multirow}
\usepackage{multicol}
\newcolumntype{P}[1]{>{\centering\arraybackslash}p{#1}}

\newtheorem{theorem}{Theorem}

\newtheorem{lemma}{Lemma}
\newtheorem{corollary}{Corollary}
\theoremstyle{definition}

\newtheorem{assumption}{Assumption}
\theoremstyle{remark}
\newtheorem{remark}{Remark}

\begin{document}

\title{SignSGD with Federated Voting}

\author{Chanho Park, \IEEEmembership{Student Member, IEEE}, H. Vincent Poor, \IEEEmembership{Life Fellow, IEEE}, \\ and Namyoon Lee, \IEEEmembership{Senior Member, IEEE}


\thanks{C. Park is with the Department of Electrical Engineering, POSTECH, Pohang 37673, South Korea (e-mail: chanho26@postech.ac.kr).}

\thanks{H. V. Poor is with the Department of Electrical and Computer Engineering, Princeton University, Princeton, NJ 08544, USA (e-mail: poor@princeton.edu).}

\thanks{N. Lee is with the School of Electrical Engineering, Korea University, Seoul 02841, South Korea (e-mail: namyoon@korea.ac.kr).}
}

\markboth{Journal of \LaTeX\ Class Files,~Vol.~14, No.~8, August~2021}%
{Shell \MakeLowercase{\textit{et al.}}: A Sample Article Using IEEEtran.cls for IEEE Journals}


\maketitle

\begin{abstract}
Distributed learning is commonly used for accelerating model training by harnessing the computational capabilities of multiple-edge devices. However, in practical applications, the communication delay emerges as a bottleneck due to the substantial information exchange required between workers and a central parameter server. SignSGD with majority voting (signSGD-MV) is an effective distributed learning algorithm that can significantly reduce communication costs by one-bit quantization. However, due to heterogeneous computational capabilities, it fails to converge when the mini-batch sizes differ among workers. To overcome this, we propose a novel signSGD optimizer with \textit{federated voting} (signSGD-FV). The idea of federated voting is to exploit learnable weights to perform weighted majority voting. The server learns the weights assigned to the edge devices in an online fashion based on their computational capabilities. Subsequently, these weights are employed to decode the signs of the aggregated local gradients in such a way to minimize the sign decoding error probability. We provide a unified convergence rate analysis framework applicable to scenarios where the estimated weights are known to the parameter server either perfectly or imperfectly. We demonstrate that the proposed signSGD-FV algorithm has a theoretical convergence guarantee even when edge devices use heterogeneous mini-batch sizes. Experimental results show that signSGD-FV outperforms signSGD-MV, exhibiting a faster convergence rate, especially in heterogeneous mini-batch sizes.
\end{abstract}


\section{Introduction} \label{sec:introduction}



\IEEEPARstart{A}{s} machine learning transitions from traditional centralized data-center configurations to operating on edge devices, innovative approaches for distributed learning are gaining attention \cite{dean2012large, bottou2010large, mcmahan2017communication}. Distributed learning aims to optimize a model for multiple local datasets stored on the edge devices, with the most typical structure being multiple devices communicating with one parameter server to aggregate the locally computed model update information \cite{mcmahan2017communication}. In each iteration, edge devices train a local model using their local data and then transmit their local models or local gradients to a central server, known as a parameter server \cite{dean2012large}. The server subsequently aggregates these incoming model updates (or local gradients) and returns its result to the edge devices. Consequently, in the implementation of distributed learning, the need for edge devices to exchange their trained model parameters or gradients is sufficient, eliminating the need to share their local data with the server. Moreover, the distributed learning system can improve its convergence rate by employing an SGD optimizer \cite{liu2020distributed}. This inherent characteristic of distributed learning fully demonstrates its practicality, and such learning algorithms have recently been utilized in various fields \cite{xu2021federated, hard2018federated}.


In practice, implementing distributed learning algorithms is a challenging task. One primary challenge is the communication bottleneck, which stems from the substantial amount of information that must be exchanged between the edge devices and the parameter server using wireline or wireless networks with limited bandwidth. In recent years, several techniques aimed at alleviating this communication burden have been developed \cite{cao2023communication, chen2021communication}. For example, one strategy for mitigating these communications costs is for the edge devices to compress their local gradients before transmission using techniques such as quantization or sparsification. Despite notable progress in addressing the communication bottleneck in distributed learning, there remains a significant gap in our understanding of the impact of computing heterogeneity across edge devices in this context.

Edge devices, such as sensors or personal computers, are inherently heterogeneous in terms of their computing capabilities. For synchronous stochastic gradient computation, for example, personal computers can use a large mini-batch size. In contrast, Internet of Things (IoT) sensors may employ a significantly smaller mini-batch size. The variance in computing capabilities among edge devices presents a notable obstacle when applying communication-efficient distributed learning algorithms that leverage compression techniques such as quantization. This challenge arises due to the diverse effects of quantization errors in the locally computed stochastic gradient across different edge devices. This variation in quantization errors significantly impacts gradient aggregation at the parameter server, causing a slowdown in convergence speed or even potential divergence. For example, signSGD with majority voting (signSGD-MV) is a simple yet communication-efficient distributed learning algorithm introduced in \cite{bernstein2018asignsgd,bernstein2018bsignsgd}. This algorithm offers a remarkable 32-fold reduction in communication costs per iteration compared to full-precision, distributed SGD. However, it was shown to diverge when the edge devices utilize varying sizes of mini-batches or for smaller mini-batch configurations \cite{karimireddy2019error}.

In this paper, we provide a comprehensive analysis of the impact of computational heterogeneity when using signSGD-MV type algorithms. In particular, we present signSGD with \textit{federated voting} (signSGD-FV) as a communication-efficient and robust distributed learning method for computational heterogeneity. The main idea of signSGD-FV is to employ learnable weights when performing majority voting for gradient aggregation. Unlike signSGD-MV, which uses equal weights, our approach uses trainable weight parameters, adapting during learning to prioritize edge devices with high computational capabilities, particularly those using large mini-batch sizes. Our main finding is that signSGD-FV has a theoretical convergence guarantee, even with heterogeneous mini-batch sizes and imperfect knowledge of the weights under some conditions. Using experiments, we verify that signSGD-FV outperforms the existing distributed learning algorithms using heterogeneous batch sizes across edge devices.

\subsection{Related Works}

{\bf Gradient compression:}
As noted above, one representative lossy compression method is the gradient quantization, which includes scalar quantization \cite{seide20141, alistarh2017qsgd, wen2017terngrad, bernstein2018asignsgd, lee2020bayesian, park2021bayesian}, vector quantization \cite{shlezinger2020uveqfed, gandikota2021vqsgd, eldar2022machine}, gradient differential coding \cite{mishchenko2019distributed, gorbunov2021marina} and adaptive quantization techniques \cite{jhunjhunwala2021adaptive, honig2022dadaquant}. Gradient sparsification is an alternative approach introduced in \cite{wangni2018gradient}. The underlying idea of sparsification is to select the largest components in magnitude of the locally computed stochastic gradient for transmission. Selection criteria include Top-K SGD \cite{shi2019understanding}, Sparsified-SGD \cite{stich2018sparsified}, and some variations \cite{rothchild2020fetchsgd, sahu2021rethinking, shanbhag2018efficient, li2022near, lin2017deep, han2020adaptive}. Recently, advanced lossy compression techniques have been proposed by jointly applying quantization and sparsification, examples of which include TernGrad \cite{wen2017terngrad}, Qsparse-local-SGD \cite{basu2019qsparse}, sparse ternary compression \cite{sattler2019robust}, and sparse-signSGD (${\sf S}^3$GD) \cite{park2023s, park2023sparse}.

{\bf Computational heterogeneity:}
The computational heterogeneity inherent in the edge devices in some applications presents a formidable challenge when executing distributed optimization for these applications \cite{pfeiffer2023federated}. In scenarios where synchronous aggregation of local gradients or models is employed, the server must await the completion of local updates by workers with limited computational capabilities before initiating the aggregation process. This waiting period can introduce significant delays in the overall optimization process. Numerous strategies have been proposed to address this issue, with the asynchronous distributed learning approach with multiple local updates emerging as the most widely adopted method \cite{recht2011hogwild, zhang2015deep, xu2023asynchronous, mcmahan2017communication, wang2021novel}. The most relevant prior studies involve the utilization of adaptive mini-batch sizes tailored to the computing capabilities of individual edge devices \cite{ferdinand2020anytime, shi2022talk, park2022amble, ma2023adaptive, liu2023dynamite}. While this approach effectively manages computational diversity, it necessitates high communication costs for transmitting high-dimensional vectors without implementing quantization. However, when quantization is applied to mitigate communication costs in heterogeneous settings, many existing algorithms can fail to converge to local optima due to the distinct errors introduced by the quantization process.

{\bf Weighted majority voting:} Weighted majority voting has been applied in diverse applications, such as communication systems \cite{hong2017weighted, kim2019supervised}, crowdsourcing \cite{li2014error, kim2023worker}, and ensemble learning \cite{berend2015finite, kim2023distributed}. Notably, in communication systems, weighted majority voting is an optimal decoder, particularly in scenarios where a transmitter sends binary information with repetition codes over parallel binary symmetric channels (BSCs). The optimal weights are determined by the log-likelihood ratio (LLR) and are obtained as a function of cross-over probabilities for the BSCs \cite{jeon2018one, kim2019supervised}. In federated learning, weighted majority voting has become an effective strategy for enhancing model aggregation performance. This approach addresses the heterogeneity of data distributions \cite{mcmahan2017communication, wu2021fast, li2023revisiting, chen2020focus, guo2020secure}. For example, in \cite{wu2021fast}, the parameter server assigns weights based on the cosine similarity between the previously aggregated gradient and the locally computed gradient of each edge device. Notably, while effective, these approaches often overlook the joint consideration of gradient quantization effects and the underlying majority voting principle.

\subsection{Contributions}

This paper considers a distributed learning problem in which edge devices collaboratively optimize neural network model parameters by communicating with a shared parameter server. In each iteration, edge devices compute stochastic gradient vectors with varying mini-batch sizes tailored to their computing capabilities, facilitating synchronous updates. Subsequently, these devices perform one-bit quantization to compress the gradient information, which is then transmitted to the parameter server through a band-limited communication channel. The parameter server aggregates the binary stochastic gradient vectors using a weighted majority voting rule. In this problem, our main contribution is to provide a comprehensive analysis of the impact of gradient quantization alongside the principles of weighted majority voting. We summarize our main contribution as follows:

\begin{itemize}

    \item We propose a novel perspective on the gradient sign aggregation challenge in distributed signSGD with varying batch sizes. Our approach is to treat the binary source aggregation problem as a decoding problem for a repetition code over parallel BSCs with distinct cross-over probabilities. We first present an optimal weighted majority voting aggregation, assuming perfect knowledge of cross-over probabilities at the server. However, obtaining such perfect knowledge is challenging in practice. To address this, we introduce \textit{federated voting} (FV), a variant of weighted majority voting with learnable weights. We term this approach signSGD-FV. The key concept behind signSGD-FV lies in executing a weighted majority voting process. In this approach, the weights undergo recursive updates, driven by comparing the local gradient signs transmitted from the edge devices and the global gradient signs decoded by the server. This adaptability empowers the weights to assign greater importance to workers with high computational capabilities, particularly those employing large mini-batch sizes.

    \item We present a unified convergence analysis for signSGD with an arbitrary binary decoding rule. Our analysis framework reveals that the convergence rate of signSGD with weighted majority voting can be accelerated as the decoding error probability of the aggregated sign information decreases, aligning with our intuitive expectations. This analysis establishes an upper bound on the decoding error probability of signSGD when federated voting is applied. Our primary finding is that signSGD-FV provides a theoretical guarantee of convergence, even when workers employ heterogeneous mini-batch sizes.

    \item Through simulations, we show the enhanced performance of signSGD-FV compared to existing algorithms. Our simulations involve training convolutional neural networks (CNNs) and ResNet-56 for image classification tasks, employing the MNIST, CIFAR-10, and CIFAR-100 datasets. The results highlight the superiority of signSGD-FV over signSGD-MV, particularly in scenarios where the mini-batch sizes of workers vary. Furthermore, we compare the signSGD-FV algorithm with conventional SGD-type algorithms using full precision. Our findings demonstrate that signSGD-FV achieves a remarkable 32x reduction in communication costs while preserving convergence properties, even when dealing with heterogeneous mini-batch sizes across edge devices.

\end{itemize}

\section{Preliminaries}
In this section, we briefly review the existing signSGD-MV algorithm and a related algorithm, signGD-MV, which helps to understand an upper bound on the convergence performance of signSGD-MV.

\subsection{SignSGD-MV}
We first review the signSGD-MV algorithm in a distributed learning problem. Let $\mathbf{x}=[x_1,x_2,\ldots,x_N] \in \mathbb{R}^N$ be a model parameter of a deep neural network (DNN) of dimension $N$. The mathematical formulation of machine learning in this case can be defined as the following optimization problem:
\begin{align}
    \min_{\mathbf{x} \in \mathbb{R}^N} f(\mathbf{x}) := \mathbb{E}_{\mathbf{d} \sim \mathcal{D}} \left[ F (\mathbf{x}; \mathbf{d}) \right], \label{eq:optimization}
\end{align}
where $\mathbf{d}$ is a data sample randomly drawn from a dataset $\mathcal{D}$, and $F(\cdot)$ is a task-dependent empirical loss function, which can be non-convex. In a distributed learning scenario with $M$ workers, the optimization problem in \eqref{eq:optimization} simplifies to
\begin{align}
    \min_{\mathbf{x} \in \mathbb{R}^N} f(\mathbf{x}) :=\frac{1}{M} \sum_{m=1}^M f_m({\bf x}),\label{eq:dis_optimization}
\end{align}
where $f_m({\bf x})=\mathbb{E}_{\mathbf{d} \sim \mathcal{D}_m} \left[ F (\mathbf{x}; \mathbf{d}) \right]$ is a local loss function for worker $m\in [M]$. This local loss function is computed by using data samples from a local data set $\mathcal{D}_m$ assigned to worker $m$, which is a subset of a global data set $\mathcal{D}=\bigcup_{m=1}^M\mathcal{D}_m$.

{\bf Stochastic gradient computation:} At iteration $t$, worker $m \in [M]$ computes a local stochastic gradient using the model parameter ${\bf x}^t$ and randomly selected mini-batch data samples $\mathcal{B}_m^t\subset \mathcal{D}_m$ as 
\begin{align}
    {\bf g}^t_m:= \frac{1}{B_m} \sum_{\mathbf{d} \in \mathcal{B}_m^t} \nabla F \left( \mathbf{x}^t; \mathbf{d} \right) \in \mathbb{R}^N,\label{eq:SG}
\end{align}
where $B_{m}=|\mathcal{B}_m^t|$ is a fixed mini-batch size of worker $m$.

{\bf Sign quantization:} Worker $m\in [M]$ performs one-bit gradient quantization to reduce the communication cost. Let $g_{m,n}^t$ be the $n$th coordinate of ${\bf g}^t_m$. Then, the sign of $g_{m,n}^t$ is computed as
\begin{align}
    Y_{m,n}^t = {\sf sign}\left(g_{m,n}^t\right)\in \{-1,+1\}, \label{eq:signSG}
\end{align}
for $n\in [N]$. This sign information is sent to the parameter server for aggregation. We assume here that communication between the workers and the server is error-free.  Note that this assumption is frequently embraced within the context of wireline network configurations, as demonstrated in \cite{seide20141, bernstein2018bsignsgd, alistarh2017qsgd, wen2017terngrad, bernstein2018asignsgd, lee2020bayesian, park2021bayesian}. 


{\bf Majority voting aggregation and model update:} The parameter server applies the majority voting principle to aggregate the locally computed stochastic sign information. This aggregated sign is computed as
\begin{align} \label{eqn:MV_rule}
  {\hat U}_{n}^t = {\sf sign}\left[ \sum_{m=1}^M  Y_{m,n}^t\right] \in \{-1,+1\},
\end{align}
and these aggregated signs are sent to all workers. Then, each worker updates its model parameter as
\begin{align}
	x_n^{t+1} = x_n^t - \delta {\hat U}_{n}^t, \,\, \forall n \in [N],
\end{align}
where $\delta$ is a fixed learning rate parameter.

\subsection{An Upper Bound on SignSGD-MV}

It is instructive to consider some ideal cases to understand the limits of signSGD-MV convergence speed. The ideal cases to attain the highest convergence rate of signSGD-MV involves full-data knowledge and full-batch computation as listed below:

\begin{itemize}
\item {\bf Case 1 (Full-data knowledge):} Each edge device has perfect knowledge of the global dataset $\mathcal{D}=\bigcup_{m=1}^M\mathcal{D}_m$.
   \item {\bf Case 2 (Full-batch computation):} Each edge device is capable of computing the local gradient with full-batch size, i.e., $B_m = |\mathcal{D}_m|$.
\end{itemize}

Under Cases 1 and 2, each local device can compute the true gradient using the model parameter $\mathbf{x}^t$ at iteration $t \in [T]$ by using all data points in the dataset $\mathcal{D}$ without any limitation on the data knowledge and compute capability. As a result, under these ideal cases, it is sufficient to consider a single-edge device to understand the limits on the convergence rate. Let $ {\bf \bar g}^t=\left[{\bar g}_{1}^t,{\bar g}_{2}^t,\ldots, {\bar g}_{N}^t\right]$ be the true gradient at iteration $t$, which is defined as
\begin{align}
{\bf \bar g}^t = \frac{1}{|\mathcal{D}|} \sum_{\mathbf{d} \in \mathcal{D}} \nabla F \left( \mathbf{x}^t; \mathbf{d} \right) \in \mathbb{R}^N.
\end{align}
Then, the sign of the true $n$th gradient coordinate is denoted by
\begin{align}
U_n^t= {\sf sign}\left( {\bar g}_{n}^t\right) \in \{-1,+1\}. \label{eq:idealsign}
\end{align}
Unlike ${\hat U}_n^t$ in \eqref{eqn:MV_rule}, which is the estimated sign of the gradient using majority vote, $U_n^t$ is the true sign of the gradient, and it is deterministic because  $  \frac{1}{|\mathcal{D}|} \sum_{\mathbf{d} \in \mathcal{D}} \nabla F \left( \mathbf{x}^t; \mathbf{d} \right) $ produces a deterministic value for given $\mathbf{x}^t$ and  $\mathcal{D}$. Using this true sign information, the server updates the model as
\begin{align}
x_n^{t+1} = x_n^t - \delta { U}_{n}^t, \,\, \forall n \in [N],
\end{align}
where $\delta$ is a fixed learning rate parameter. This algorithm is often referred to as sign gradient descent (signGD) \cite{balles2018dissecting, ma2023understanding}, and its performance provides an upper-bound on the performance of signSGD-MV because it updates the model with correct gradient signs in every iteration.

\begin{figure}[t]
    \centering
    \includegraphics[width=1\columnwidth]{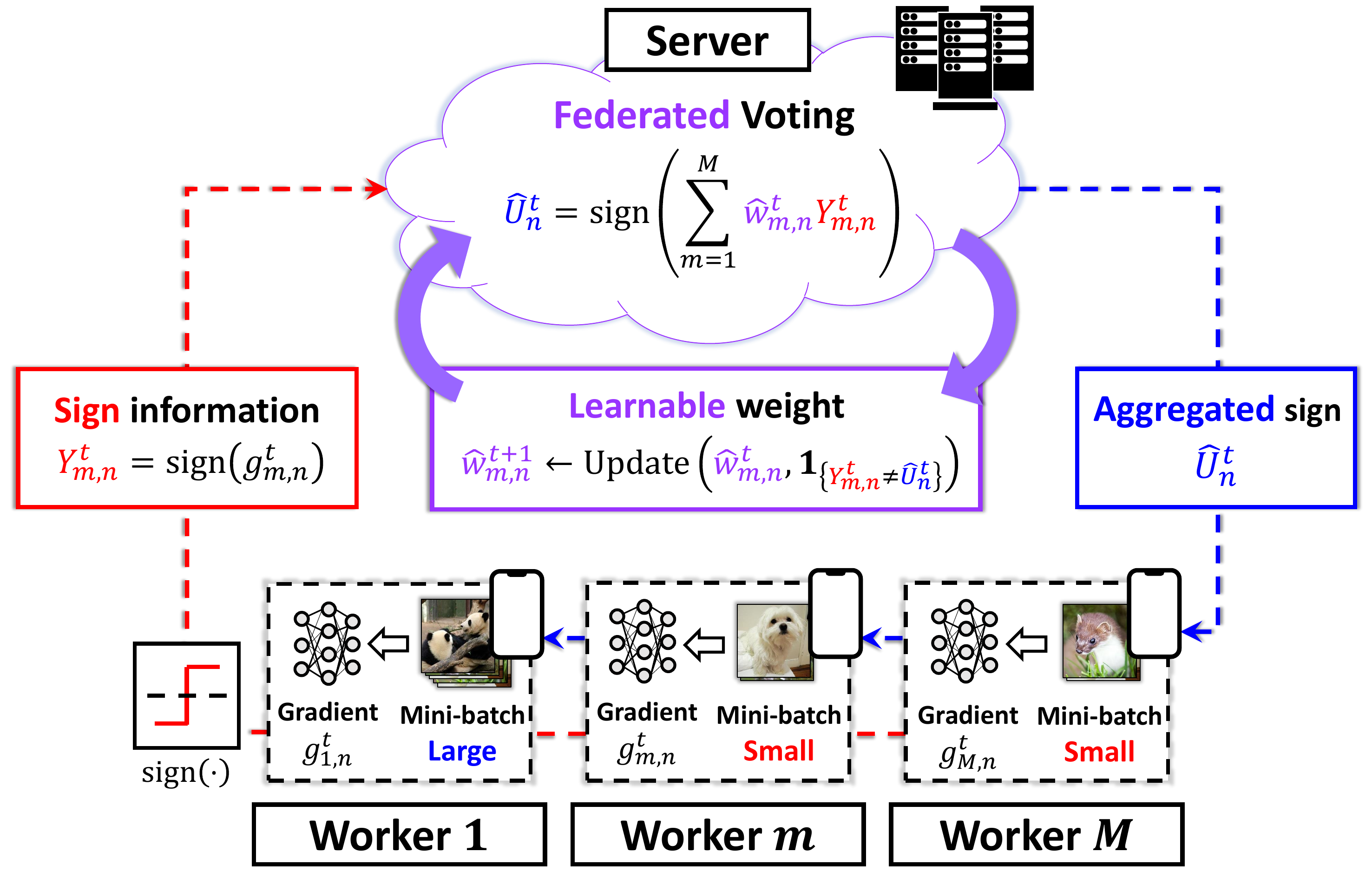}
    \caption{An illustration of signSGD-FV.}
    \label{fig:DL_system}
\end{figure}

\section{SignSGD-FV}
We begin by offering a novel perspective on the gradient sign aggregation challenge within the context of distributed signSGD, specifically addressing scenarios involving heterogeneous batch sizes. Our approach involves examining this problem through the lens of a decoding problem associated with a repetition code over parallel BSCs characterized by distinct cross-over probabilities. By leveraging this interpretation, we introduce an optimal solution for weighted majority voting aggregation in distributed signSGD. This solution assumes a scenario where the server possesses perfect knowledge of the cross-over probabilities. However, acquiring such perfect knowledge is challenging in real-world applications. In light of this challenge, we propose an alternative approach termed \textit{federated voting}, which involves weighted majority voting with learnable weights. The weights are updated per iteration by comparing the local gradient sign information that each device sent and the estimated gradient sign decoded by harnessing the previously updated weight. This adaptive mechanism addresses the limitation of relying on perfect knowledge and enhances the robustness of the aggregation process in distributed signSGD. Fig. \ref{fig:DL_system} illustrates the entire signSGD-FV algorithm.

\subsection{Algorithm}
{\bf Distributed repetition encoding:} 
Under full-data knowledge and full-batch computation, each worker can compute the true sign of the gradient $U_{n}^t$ for $n\in [N]$. We define this sign information per coordinate as a \textit{message bit} for the $n$th gradient coordinate at iteration $t$. We presume that this message bit is virtually encoded across the number of workers with repetition coding with the rate of $R=\frac{1}{M}$. The encoder output is defined as a codeword with length $M$, namely, 
\begin{align}
    {\bf U}_n^t=\left[ U_{n}^t, \cdots, U_{n}^t \right] \in \{-1,+1\}^M,
\end{align}
where all coded bits are identical with the true sign.

{\bf Communication channel model for stochastic gradient computation:}
From \eqref{eq:SG}, worker $m\in [M]$ calculates the stochastic gradient $g_{m,n}^t$ using a mini-batch size of $B_m$. It is important to note that the sign of $g_{m,n}^t$, represented as $Y_{m,n}^t={\sf sign}\left(g_{m,n}^t\right)$ in  \eqref{eq:signSG}, may differ from the true sign of the ideal gradient as defined in \eqref{eq:idealsign}, i.e., $U_{n}^t={\sf sign}\left({\bar g}_n^t\right)$. This sign discrepancy arises due to the limited batch size employed by each worker when computing the stochastic gradient. We model the process of stochastic gradient computation with limited data points and one-bit quantization as a noisy communication channel, specifically, a BSC. As illustrated in Fig. \ref{fig:BSC_model}, from the true gradient message $U_n^t$, each worker $m \in [M]$ computes a one-bit stochastic gradient message $Y_{m,n}^t$ for coordinate $n$ at iteration $t$ via the corresponding BSC with cross-over probability of $p_{m,n}^t$, and then sends it. Consequently, the server receives $M$ independent noisy binary gradient message bits, i.e., $Y_{1,n}^t, \cdots, Y_{M,n}^t$, assuming that all workers employ independent and identically distributed (IID) samples when computing the stochastic gradient. In this modeling method, the cross-over probability of the BSC, $ p_{m,n}^t$ is defined as
\begin{align}
    p_{m,n}^t = \mathbb{P} \left[ U_{n}^t \neq Y_{m,n}^t \right]. \label{eqn:def_cross}
\end{align}
This cross-over probability is a crucial parameter inversely proportional to the batch size $B_m$. For instance, when full-batch size is used for the gradient computation, the probability $p_{m,n}^t$ approaches zero, i.e., no sign flip error occurs. This observation underscores the critical role of error probability in developing optimal aggregation methods, especially when the workers employ heterogeneous batch sizes.

\begin{figure}[t]
    \centering
    \includegraphics[width=1\columnwidth]{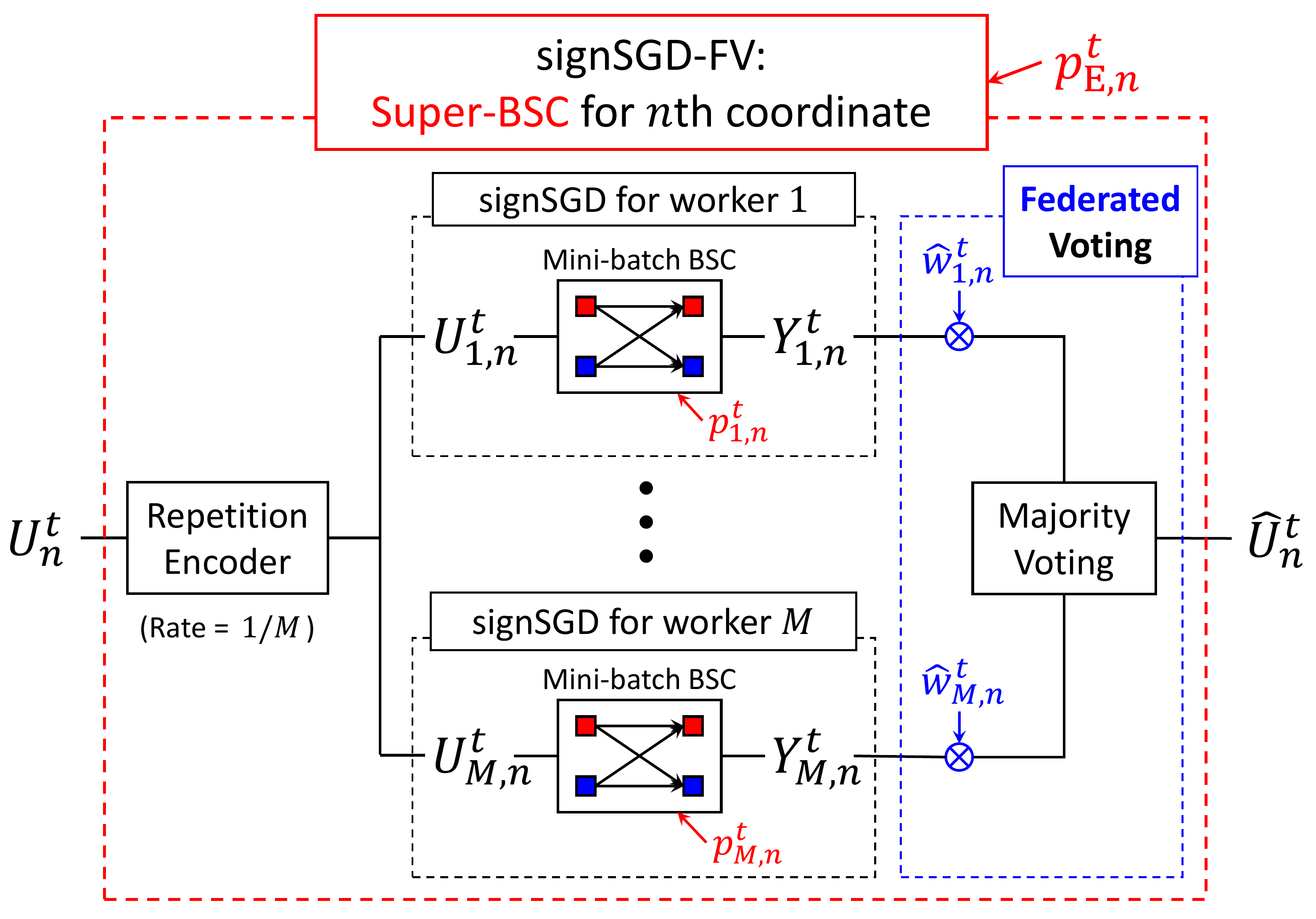}
    \caption{The coding-theoretic interpretation of signSGD-FV.}
    \label{fig:BSC_model}
\end{figure}

{\bf Optimal aggregation of the sign of gradients:}  
The main task of the parameter server is to optimally aggregate the $M$ stochastic gradients $Y_{1,n}^t, \cdots, Y_{M,n}^t$ to recover the one-bit gradient message $U_n^t$ for $n \in [N]$. From our communication model, it becomes evident that determining the most effective aggregation method is equivalent to identifying the optimal decoder for a repetition code with a rate of $\frac{1}{M}$ under parallel BSCs, each with cross-over probabilities $p_{1,n}^t, \cdots, p_{M,n}^t$.

This interpretation from coding theory provides valuable insight into the development of an optimal aggregation method, seen through the lens of the classical maximum likelihood (ML) decoder. To design the ML decoder, we first define the log-likelihood function when $U_n^t = +1$ as
\begin{align}
    &\ln \mathbb{P} \left[ \left. Y_{1,n}^t, \cdots, Y_{M,n}^t \right| U_{n}^t = +1 \right] \nonumber \\
    &= \! \sum_{m=1}^M \! \ln \left( 1 \!-\! p_{m,n}^t \right) {\bf 1}_{\left\{ Y_{m,n}^t \! = +1 \right\}} \!+\! \ln  p_{m,n}^t  {\bf 1}_{\left\{ Y_{m,n}^t \! = -1 \right\}}. \label{eqn:llr_plus}
\end{align}
Similarly, the log-likelihood function when $U_n^t = -1$ is given by
\begin{align}
    &\ln \mathbb{P} \left[ \left. Y_{1,n}^t, \cdots, Y_{M,n}^t \right| U_{n}^t = -1 \right] \nonumber \\
    &= \sum_{m=1}^M \! \ln \left( 1 \!-\! p_{m,n}^t \right) {\bf 1}_{\left\{ Y_{m,n}^t \! = -1 \right\}} \!+\! \ln p_{m,n}^t {\bf 1}_{\left\{ Y_{m,n}^t \! = +1 \right\}}. \label{eqn:llr_minus}
\end{align}
Using \eqref{eqn:llr_plus} and \eqref{eqn:llr_minus}, the LLR is given by
\begin{align}
    &\ln \frac{ \mathbb{P} \left[ \left. Y_{1,n}^t, \cdots, Y_{M,n}^t \right| U_{n}^t = +1 \right]}{ \mathbb{P} \left[ \left. Y_{1,n}^t, \cdots, Y_{M,n}^t \right| U_{n}^t = -1 \right]} 
  = \sum_{m=1}^M \ln \frac{1 - p_{m,n}^t}{p_{n,m}^t} Y_{m,n}^t.
\end{align}
Since the ML decoding rule selects the case where the likelihood value is larger between $U_n^t = +1$ and $U_n^t = -1$ cases, it is optimal to adopt the sign of LLR. Therefore, the optimal aggregation method for arbitrary batch-sizes across workers is
\begin{align}
    {\hat U}_{{\sf WMV}, n}^t = {\sf sign} \left( \sum_{m=1}^M w_{m,n}^t Y_{m,n}^t \right), \label{eqn:WMV_agg}
\end{align}
where $w_{m,n}^t = \ln \frac{1 - p_{m,n}^t}{p_{m,n}^t}$. We refer to this aggregation as a \textit{weighted majority vote (WMV)} decoder. The decoding error probability is defined as
\begin{align}
    P_{\mathsf{E}, n}^t=\mathbb{P} \left[ U_n^t \ne {\hat U}_{{\sf WMV}, n}^t \right].
\end{align}
For simplicity, we omit the notation $n$ and $t$ from $P_{\mathsf{E}, n}^t$ in the sequel. This decoding error probability can be interpreted as a cross-over probability for the $n$th coordinate super-channel as depicted in Fig. \ref{fig:BSC_model}. This will serve as a pivotal parameter for evaluating the convergence rate in subsequent analyses. Following the execution of optimal gradient aggregation, the parameter server proceeds to update the model as follows:
\begin{align} \label{eqn:update_model}
    x_n^{t+1} = x_n^t - \delta {\hat U}_{{\sf WMV}, n}^t, \,\, \forall n \in [N].
\end{align} 
The optimality of WMV decoding is described in our convergence analysis of Sec. \ref{sec:converge}, which shows that the convergence rate of signSGD-style algorithms can be accelerated by reducing the sign decoding error probability.

  
\begin{algorithm}[t]
\caption{signSGD-FV}
\label{alg:signSGD-FV}
\begin{algorithmic}
    \STATE \textbf{Input:} Initial model $\mathbf{x}^1$, learning rate $\delta$, the number of workers $M$, worker $m$'s mini-batch size $B_m$, total iteration $T$, initial weights $\hat{w}_{m,n}^1 \! = \! 1$, initial phase duration $T_\mathsf{in}$
    
    \FOR{$t=1:T$}
        \vspace{0.5em}
        \STATE {\bf *** Each worker ***}
        \FOR{$m=1: M$}
            \STATE {\bf Compute} $\mathbf{g}_m^t$ with batch size $B_m$
            \STATE {\bf Encode} $Y_{m,n}^t = \mathsf{sign} \left( g_{m,n}^t \right), \, \forall n \in [N]$
            \STATE {\bf Send} $Y_{m,1}^t, \cdots, Y_{m,N}^t$ to {\bf server}
        \ENDFOR

        \vspace{0.5em}
        \STATE {\bf *** Server ***}
        \FOR{$n=1: N$}
            \STATE {\bf Decode} \\
            \vspace{0.5em}
            $\hat{U}_n^t = 
            \begin{cases}
            \mathsf{sign} \left( \sum_{m=1}^M \hat{w}_{m,n}^t Y_{m,n}^t \right) , & \text{if } t > T_\mathsf{in} \\
            \mathsf{sign} \left( \sum_{m=1}^M Y_{m,n}^t \right), & \text{o.w.}
            \end{cases}$
            \vspace{0.1em}
            \FOR{$m = 1: M$}
                \STATE {\bf Estimate }$\hat{p}_{m,n}^{t+1} = \frac{\sum_{i=1}^t \mathbf{1}_{\left\{ Y_{m,n}^t \ne \hat{U}_n^t \right\}}}{t}$
                \STATE {\bf Update} $\hat{w}_{m,n}^{t+1} = \ln \frac{1 - \hat{p}_{m,n}^{t+1}}{\hat{p}_{m,n}^{t+1}}$
            \ENDFOR
            \STATE {\bf Send} $\hat{U}_n^t$ to {\bf all workers} $m \in [M]$
        \ENDFOR

        \vspace{0.5em}
        \STATE {\bf *** Each worker ***}
        \FOR{$m=1: M$}
            \STATE {\bf Update} $x_n^{t+1} = x_n^t - \delta \hat{U}_n^t, \, \forall n \in [N]$
        \ENDFOR
        \vspace{0.5em}
    \ENDFOR
\end{algorithmic}
\end{algorithm}

\subsection{Learning cross-over probabilities} \label{sec:cross-over}
When performing the optimal aggregation, the server needs to know the cross-over probabilities of the BSCs for every coordinates, i.e., $p_{1,n}^t, \cdots, p_{M,n}^t$ for $n \in [N]$. Practically, attaining perfect knowledge of these probabilities is infeasible. To resolve this problem, we introduce an approach for estimating the cross-over probabilities by leveraging a joint channel estimation and data detection method, originally introduced in the context of wireless communications \cite{kim2019supervised}.

The proposed approach is to empirically estimate the cross-over probabilities by comparing the signs of the received bits and the decoded bits. In the initial phase of our iterative process, $t \leq T_{\sf in}$, the server encounters a challenge in executing optimal weighted majority voting due to the lack of knowledge about channel cross-over probabilities. In this phase, the server decodes the message bit through simple majority vote decoding \cite{bernstein2018asignsgd} as
\begin{align}
   {\hat U}_{{\sf MV}, n}^t = {\sf sign} \left( \sum_{m=1}^M   Y_{m,n}^t \right).
\end{align}
Then, the server computes the empirical average of the bit differences between the received gradient bit $Y_{m,n}^t$ and the decoded bit ${\hat U}_{{\sf MV}, n}^t$ for each coordinate $n \in [N]$. This empirical average provides the estimate of the cross-over probability as 
\begin{align}
    {\hat p}_{m,n}^{T_{\sf in}} = 
    \frac{\sum_{i=1}^{T_{\sf in}}{\bf 1}_{ \left\{ Y_{m,n}^i \neq  {\hat U}_{{\sf MV}, n}^i \right\}} }{T_{\sf in}}.
\end{align} 
After the initial phase $t > T_{\sf in}$, the server begins using the estimated cross-over probabilities established during the initial phase. These probabilities are applied in the process of decoding the one-bit gradient message, employing the weighted majority vote decoding rule denoted as ${\hat U}_{{\sf WMV}, n}^t = {\sf sign} \left( \sum_{m=1}^M \ln \frac{1- {\hat p}_{m,n}^{t}}{ {\hat p}_{m,n}^{t}} Y_{m,n}^t \right)$. Subsequently, the server proceeds to update these cross-over probabilities in a online manner as 
\begin{align} \label{eqn:prob_est}
      {\hat p}_{m,n}^{t+1} = \frac{T_{\sf in}}{t} {\hat p}_{m,n}^{T_{\sf in}} +  \frac{t-T_{\sf in}}{t}
    \frac{\sum_{i=T_{\sf in}+1}^{t} {\bf 1}_{ \left\{  Y_{m,n}^i \neq {\hat U}_{{\sf WMV} ,n}^i \right\}} }{t-T_{\sf in}}.
\end{align}
Using the estimated cross-over probabilities, it also updates the weights for the next iteration decoding as
\begin{align}
    {\hat w}_{m,n}^{t+1} =\ln \frac{1 - {\hat p}_{m,n}^{t+1}}{{\hat p}_{m,n}^{t+1}}. \label{eqn:est_weight}
\end{align}
The entire procedure is summarized in Algorithm \ref{alg:signSGD-FV}.

\section{Unified Convergence Rate Analysis} \label{sec:converge}
In this section, we provide a unified convergence rate analysis for signSGD-style algorithms with an arbitrary binary aggregation method. Leveraging this unified convergence rate analysis framework, we provide the convergence rate of signSGD algorithms when using a weighted majority voting decoding method by capitalizing on the effect of heterogeneous batch sizes across workers.

\subsection{Assumptions}


We first introduce some assumptions that serve as a preparatory step toward the presentation of the convergence rate analysis in the following subsection.

\begin{assumption}[Lower bound] \label{ass:1}
    For all $\mathbf{x} \in \mathbb{R}^N$ and a global minimum point $\mathbf{x}^\star$, we have an objective value as
    \begin{align}
        f (\mathbf{x}) \ge f \left( \mathbf{x}^\star \right) = f^\star.
    \end{align}
\end{assumption}

Assumption \ref{ass:1} is necessary for the convergence to local minima. 

\begin{assumption}[Coordinate-wise smoothness] \label{ass:2}
    For all $\mathbf{x}, \mathbf{y} \in \mathbb{R}^N$, there exists a vector with non-negative constants $\mathbf{L} = \left[ L_1, \cdots, L_N \right]$ that satisfies 
    \begin{align}
        \left| f(\mathbf{y}) \! - \! f(\mathbf{x}) \! - \! \left\langle \nabla f(\mathbf{x}), \mathbf{y} - \mathbf{x} \right\rangle \right| \le \sum_{n=1}^N \frac{L_n}{2} \left( y_n - x_n \right)^2.
    \end{align}
\end{assumption}
Assumption \ref{ass:2} indicates that the objective function satisfies a coordinate-wise Lipschitz condition. Assumptions \ref{ass:1} and \ref{ass:2} are commonly used for the convergence analysis of learning algorithms.

\begin{assumption}[Unbiased stochastic gradient with a finite variance]
\label{ass:3}
    The stochastic gradient of worker $m \in [M]$, $\mathbf{g}_m^t$, is unbiased and each coordinate $n \in [N]$ of $\mathbf{g}_m$ has a normalized variance bound with a constant $\sigma \in \mathbb{R}^+$, i.e.,
    \begin{align}
        \mathbb{E} \left[ \mathbf{g}_m \right] = \bar{\mathbf{g}}, \,\,\, \mathbb{E} \left[ \frac{\left( g_{m,n} - \bar{g}_n \right)^2}{\left| \bar{g}_n \right|^2} \right] \le \sigma^2.
    \end{align}
\end{assumption}
Assumption \ref{ass:3} ensures that the stochastic gradients are unbiased across workers, and the normalized variance of all components for the stochastic gradient are bounded by a constant. This normalization in Assumption \ref{ass:3} helps to derive an upper-bound on cross-over probability. The following analysis is based on these three assumptions.

\subsection{Convergence Rate Analysis}

We provide a unified framework for analyzing the convergence rate of signSGD using an arbitrary binary aggregation method. The following theorem is the main result of this section.

\begin{theorem}[Universal convergence rate] \label{thm:1}
    Let $A_n^t:\{-1,+1\}^M\rightarrow \{-1,+1\}$ be a binary sign aggregation function applied to the $n$th gradient component at iteration $t$. This binary sign aggregation function produces an estimate of the true gradient sign $U_n^t$, i.e., 
    \begin{align}
        {\hat U}_n^t = A_n^t \left( Y_{1,n}^t, \cdots, Y_{M,n}^t \right) \in \{-1,+1\}.
    \end{align}
    Using the estimated gradient sign $\hat{U}_n^t$, the maximum of the sign decoding error probability over all coordinates and iterations is denoted by
    \begin{align}
        P_{\sf E}^{\sf max}  = \underset{n \in [N], \, t \in [T]}{\max} \,\, \mathbb{P} \left[ U_n^t \ne {\hat U}_n^t \right].
    \end{align}
    Then, with a fixed learning parameter $\delta = \sqrt{\frac{2 \left( f^1 - f^\star \right)}{T \lVert \mathbf{L} \rVert_1}}$, the convergence rate of signSGD with the binary aggregation function $A_n^t (\cdot)$ is bounded as
    \begin{align}
        \mathbb{E} \left[ \frac{1}{T} \sum_{t=1}^{T} \lVert \bar{\mathbf{g}}^t \rVert_1 \right] \le \frac{1}{1 - 2 P_{\sf E}^{\sf max} } \sqrt{\frac{2 \left( f^1 - f^\star \right) \lVert \mathbf{L} \rVert_1}{T}}, 
    \end{align}
    for $P_{\sf E}^{\sf max}<\frac{1}{2}$.
\end{theorem}

\begin{proof}
    See Appendix \ref{sec:proof_thm1}.
\end{proof}

The convergence rate established in Theorem \ref{thm:1} is universally applicable across various binary decoding strategies, as it hinges on the determination of the maximum decoding error probability $P_{\sf E}^{\sf max}$. This universality stems from the fact that the rate is contingent upon the highest decoding error probability, offering a broad generalization. Consequently, our findings extend the scope of existing convergence rate analyses that utilize majority voting aggregation, as presented in \cite{bernstein2018asignsgd, safaryan2021stochastic}. For instance, Theorem \ref{thm:1} demonstrates that signSGD with binary aggregation has a guaranteed converge rate of $\mathcal{O} \left( \frac{1}{\sqrt{T}} \right)$, which is an identical rate with a standard distributed SGD algorithm \cite{zinkevich2010parallelized, meng2019convergence}.

The fundamental insight derived from Theorem \ref{thm:1} revolves around the potential for achieving a more rapid convergence rate. First of all, it is notable that the convergence rate of signGD is limited to $\mathcal{O} \left( \frac{1}{\sqrt{T}} \right)$, which is the same rate as other signSGD-based algorithms, even though signGD can completely eliminate the sign decoding errors, i.e., $P_\mathsf{E}^\mathsf{max} = 0$. Nevertheless, sign decoding errors can occur in signSGD due to the randomness of mini-batches, making the rate lower. Hence, the convergence rate acceleration of signSGD-based algorithms is contingent upon reducing the maximum decoding error probability. Leveraging this pivotal observation, our attention will be directed toward establishing an upper bound for the decoding error probability associated with applying the WMV decoding method. This upper bound will delineate the convergence rates of signSGD algorithms, tailoring them to the decoding methods and providing insight into the reasons behind the enhanced convergence rate facilitated by WMV decoding, especially in heterogeneous batch sizes.

\subsection{Decoding Error Probability Bounds for WMV Aggregation} \label{sec:error_bound}

We commence by providing a lemma that is essential to derive the upper bounds on the decoding error probabilities.

\begin{lemma}[Large deviation bound] \label{lem:1}
    For all $p \in [0,1]$ and $|t|<\infty$, 
    \begin{align}
        (1-p)e^{-tp} + pe^{t(1-p)} \leq \exp \left( \frac{1 - 2p}{4 \ln \frac{1-p}{p}} t^2 \right).
    \end{align}
\end{lemma}

\begin{proof}
    See \cite{kearns2013large}.
\end{proof}

This lemma provides insight time into the features affecting weighted sums of binary random variables. It will be useful in the proof of the following theorem.

\begin{theorem}[Decoding error bound of WMV aggregation] \label{thm:2}
Suppose the server performs WMV decoding with perfect knowledge of $p_{m,n}^t$, i.e., 
\begin{align}
    {\hat U}_{{\sf WMV}, n}^t 
    &= {\sf sign} \left( \sum_{m=1}^M \ln \frac{1 - p_{m,n}^t}{p_{m,n}^t} Y_{m,n}^t \right).
\end{align} 
Then, the decoding error probability is upper bounded by
\begin{align} \label{eqn:WMV_errorbound}
    P_\mathsf{E}^\mathsf{WMV} \le \exp\left(-  M \gamma^{\sf WMV}\right),
\end{align}
where $\gamma^{\sf WMV}=\frac{1}{2M}\sum_{m=1}^M \ln \frac{1 - p_{m,n}^t}{p_{m,n}^t} \left( \frac{1}{2} - p_{m,n}^t \right)>0$ is the error exponent.
\end{theorem} 

\begin{proof}
    See Appendix \ref{sec:proof_thm2}.
\end{proof}


Theorem \ref{thm:2} shows that the decoding error probability decreases exponentially with the number of workers. This result is aligned with the fact that the decoding error of a repetition code decreases exponentially with the code blocklength. As a result, increasing the number of workers has an impact analogous to increasing increasing the blocklength from a coding theoretical viewpoint. The error exponent $\gamma^{\sf WMV}$ determines how quickly the error probability decreases as the blocklength increases.

\begin{corollary} \label{cor:1}
    Suppose the computing error probability parameters $p_{m,n}^t$ are independent and distributed uniformly in $[0,a]$ with $a\leq \frac{1}{2}$. Then, the error exponent converges as follows:
    \begin{align}
        \lim_{M\rightarrow \infty} \gamma^{\sf WMV} &=\frac{1}{2} \mathbb{E}\left[ \ln \frac{1 - p_{m,n}^t}{p_{m,n}^t} \left( \frac{1}{2} - p_{m,n}^t \right)\right] \nonumber\\
        &=    \frac{1}{2a}\int_{0}^{a}  \ln \frac{1 - u}{u} \left( \frac{1}{2} - u \right) {\rm d}u 
        \nonumber\\
        &= \frac{1}{4}\left(1+(1-a)\ln\frac{1-a}{a}\right).
    \end{align}
\end{corollary}
The error exponent of WMV decoding is intricately tied to the distribution of the stochastic computing error probability, $p_{m,n}^t$. Here, the distribution of $p_{m,n}^t$ can play a pivotal role in shaping the characteristics of decoding errors as the number of workers increases. To illustrate, consider the scenario where $p_{m,n}^t$ follows a uniform distribution in the range of 0 to 1/2, represented by the parameter $a=1/2$. In this example, the asymptotic value of the error exponent becomes
\begin{align}
     \lim_{M\rightarrow \infty} \gamma^{\sf WMV}=  \frac{1}{4}.
\end{align}
As the value of $a$ decreases, the asymptotic error exponent increases, confirming that a lower computing error probability $p_{m,n}^t$ plays a crucial role in enhancing the overall performance of WMV decoding.

To provide more insight into how the heterogeneous batch size in computing the stochastic gradient affects the decoding error probability, we introduce a lemma that provides an upper bound on $p_{m,n}^t$ in terms of the batch size $B_m$. 

\begin{lemma}[Upper bound on the computing error probability] \label{lem:2}
Consider a scenario in which the worker $m \in [M]$ employs a mini-batch size of $B_m \in \mathbb{N}$ during the sign computation of the stochastic gradient, represented as $Y_{m,n}^t={\sf sign}\left({g}_{m,n}^t\right)$. The computing error probability that the computed result $Y_{m,n}^t$ does not match the true gradient sign value $U_{m,n}^t$ is upper bounded by
    \begin{align}
        p_{m,n}^t \le \frac{\sigma}{\sqrt{B_m}}.
    \end{align}
\end{lemma}

\begin{proof}
    See Appendix \ref{sec:proof_lem1}.
\end{proof}
The upper bound in Lemma \ref{lem:2} on the computing error probability for each worker exhibits an inverse proportionality to the square root of the batch size. This finding aligns with our intuitive understanding that employing larger batch sizes effectively mitigates the probability of stochastic gradient computation errors. Moreover, thanks to Assumption \ref{ass:3}, the upper bound on the computing error probability in Lemma \ref{lem:2} does not depend on $n$ or $t$, which avoids the necessity of finding the maximum $p_{m,n}^t$ by considering all coordinates and iterations. This interpretation posits the stochastic gradient computation error as analogous to the cross-over probability of an equivalent BSC, which is explained in our coding theoretic interpretation in \eqref{eqn:def_cross}.

Now, leveraging Lemma \ref{lem:2} and Theorem \ref{thm:2}, we introduce a more intuitive form of the upper bound on the WMV decoding error probability in terms of the batch size in the following corollary.

\begin{corollary}[Decoding error bound with mini-batch sizes] \label{cor:2}
Suppose $0< p_{m,n}^t <\frac{1}{4}$. Then, the WMV decoding error probability is upper bounded by 
  \begin{align}
        P_{\sf E}^{\sf WMV} \le \exp\left[-\frac{M}{8}\ln \left( \frac{3}{4} \frac{\sqrt{B_\mathsf{GM}}}{\sigma } \right) \right],
    \end{align}
where $B_\mathsf{GM} = \left( \prod_{m=1}^M B_m \right)^{\frac{1}{M}}$ is the geometric mean of workers' batch sizes. 
\end{corollary}

\begin{proof}
    See Appendix \ref{sec:proof_cor2}.
\end{proof}

The upper bound on decoding error probability in Corollary \ref{cor:2} shows a remarkable trend: increasing the geometric mean of workers' batch sizes improves decoding accuracy. This substantiates our intuitive understanding that a worker using a larger mini-batch size significantly contributes to refining decoding performance due to the essence of the geometric mean, and suggests the possibility that signSGD can overcome the limitations of having some small batch sizes. 

\begin{remark}[Decoding error bound of signSGD-MV]
    Referring to \cite{yue2022federated}, an upper bound on the decoding error probability for signSGD-MV can be also derived as
    \begin{align} \label{eqn:MV_errorbound}
        P_\mathsf{E}^\mathsf{MV} &\le \left[ 2 \bar{p}_n^t \exp \left( 1 - 2 \bar{p}_n^t \right) \right]^\frac{M}{2} \nonumber \\
        &= \exp \left( -M \gamma^\mathsf{MV} \right),
    \end{align}
    where $\gamma^\mathsf{MV} = \bar{p}_n^t - \frac{1}{2} \ln \left( 2e \bar{p}_n^t \right)$ is the error exponent of the MV decoder and $\bar{p}_n^t = \frac{1}{M} \sum_{m=1}^M p_{m,n}^t$ is the average of workers' computing error probabilities. Armed with Lemma \ref{lem:2}, we can express \eqref{eqn:MV_errorbound} in terms of mini-batch sizes as
    \begin{align}
        P_\mathsf{E}^\mathsf{MV} \le \exp \left[ -\frac{M}{2} \ln \left( \frac{1}{2e} \frac{\sqrt{B_\mathsf{AM}}}{\sigma} \right) \right],
    \end{align}
    where $B_\mathsf{AM} = \left( \frac{1}{M} \sum_{m=1}^M \frac{1}{\sqrt{B_m}} \right)^{-2}$ is the value associated with the arithmetic mean of batch sizes. The analysis reveals that the MV decoder's performance is linked to the arithmetic mean of the batch sizes. Notably, the variable $B_\mathsf{AM}$ is significantly influenced by the presence of a minimum batch size. This observation offers insight into the potential divergence of signSGD-MV in heterogeneous computing environments, particularly in scenarios where edge devices employ small batch sizes.
\end{remark}

\subsection{Decoding Error Probability Bounds with Imperfect Knowledge of $p_{m,n}^t$ } \label{sec:im}

When implementing the signSGD-FV algorithm, acquiring perfect knowledge of the computing error probability $p_{m,n}^t$ is impractical in a distributed learning system. In this subsection, we develop an upper bound on the decoding error probability when imperfect knowledge of ${p}_{m,n}^t$ is assumed.  Recall that the signSGD-FV algorithm uses the estimate of LLR weight $\hat{w}_{m,n}^t$ derived from ${\hat p}_{m,n}^t$ according to \eqref{eqn:est_weight}. To provide a tractable analysis result, we impose an additional assumption, this one on the uncertainty bounds of the LLR weights $\hat{w}_{m,n}^t$. 

\begin{assumption}[Uncertainty of LLR weights] \label{ass:4}
For every $n \in [N]$, $m \in [M]$, and $t \in [T]$, the ratio between the estimated and actual LLR weights is constrained by  
\begin{align}
   1-\delta_{\sf min} \leq \frac{\hat{w}_{m,n}^t}{w_{m,n}^t} \leq 1+\delta_{\sf max},\label{eqn:beta}
\end{align}
where $\delta_{\sf max}$ and $\delta_{\sf min}$ are positive constants. 
\end{assumption}

Leveraging this weight bound assumption, we present an upper bound on the decoding error probability of the proposed signSGD-FV.

\begin{theorem}[WMV decoding error bound with imperfect $p_{m,n}^t$] \label{thm:3}
Suppose the server performs WMV decoding with imperfect knowledge of ${\hat p}_{m,n}^t$, i.e., 
\begin{align}
    {\hat U}_{{\sf WMV}, n}^t 
    &= {\sf sign} \left( \sum_{m=1}^M \ln \frac{1 - \hat{p}_{m,n}^t}{\hat{p}_{m,n}^t} Y_{m,n}^t \right),
\end{align} 
where $1-\delta_{\sf min} \leq \frac{\hat{w}_{m,n}^t}{w_{m,n}^t} \leq 1+\delta_{\sf max}$. Then, the decoding error probability is upper bounded by
    \begin{align}
    P_{\sf E}^{\sf WMV} \le \exp\left(-  M \left(\frac{1 - \delta_{\sf min}}{1 + \delta_{\sf max}}  \right)\gamma^{\sf WMV}\right),
    \end{align}
 where   $\gamma^{\sf WMV}=\frac{1}{2M}\sum_{m=1}^M \ln \frac{1 - p_{m,n}^t}{p_{m,n}^t} \left( \frac{1}{2} - p_{m,n}^t \right)$.
\end{theorem} 
\begin{proof}
    See Appendix \ref{sec:thm3}.
\end{proof}

Theorem \ref{thm:3} clearly shows how imprecise knowledge regarding the LLR weight ${\hat w}_{n,m}^t$ significantly impacts the upper bound on the decoding error probability. To elucidate this point, it is instructive to compare the upper bound for the scenario of perfect knowledge, as established in Theorem \ref{thm:2}.

Analogous to the scenario in which perfect knowledge is assumed, the decoding error probability experiences an exponential decrease as the number of workers increases. Conversely, a noteworthy distinction emerges: the error exponent term undergoes a reduction by a factor of $\left(\frac{1 - \delta_{\sf min}}{1 + \delta_{\sf max}}\right)$ -- a value less than one. This fact shows that the error probability worsens as the estimation of the LLR weights becomes more inaccurate. In the extreme scenario of perfect LLR weight knowledge, where $\frac{1 - \delta_{\sf min}}{1 + \delta_{\sf max}} = 1$, the upper bound in Theorem \ref{thm:3} reduces to the result in Theorem \ref{thm:2}. This alignment confirms that the upper bound analysis generalizes the case of perfect knowledge of the LLR weights.

\section{Simulation Results}
In this section, we present simulation results to validate the performance of signSGD-FV on widely employed benchmark datasets.

\subsection{Settings}

The simulation task for signSGD-FV validation involves image classification using MNIST \cite{lecun1998gradient}, CIFAR-10, and CIFAR-100 \cite{krizhevsky2009learning} datasets. The MNIST dataset comprises 60,000 training images, while CIFAR-10 comprises 50,000 training images. All the datasets include 10,000 test images, and each image sample is categorized into one of 10 classes, except for CIFAR-100, which has 100 classes. We employ data augmentation techniques such as random cropping and horizontal random flipping for the images to enhance learning performance. To ensure uniformity across distributed learning system workers, we assign each worker an equal number of data samples. The data is assumed to be IID, with each worker uniformly handling images from all 10 classes.

We employ a CNN \cite{lecun1998gradient} tailored explicitly for the MNIST dataset in model training. Simultaneously, we utilize a ResNet \cite{he2016deep} model designed for the CIFAR-10 and CIFAR-100 datasets. The CNN architecture, having a parameter count of $N = 5 \times 10^5$, comprises two convolutional layers, each with a kernel size of $5 \times 5$, complemented by two fully-connected layers. On the other hand, the ResNet model adopted for this study is ResNet-56, characterized by 56 layers and a parameter count of $N = 8.5 \times 10^5$. For the hyper-parameters, we select a learning rate of $\delta = 10^{-3} \text{ and } 10^{-1}$ for optimizers based on signSGD and SGD in Table \ref{tab:benchmark}, respectively. Momentum and weight decay are deliberately excluded from our simulations.

\begin{table}[t] 
\caption{Mini-batch modes}
\vspace{-1em}
\label{tab:batchmode}
\begin{center}
\begin{small}
\begin{sc}
\begin{tabular}{P{0.2\columnwidth} P{0.25\columnwidth} P{0.25\columnwidth}}
\toprule
\multirow{3}{*}{Batch mode} & \multicolumn{2}{c}{\# of workers} \\
\cline{2-3}
 & Small size & \multirow{2}{*}{Large size} \\
 & $\left( B_m = 4 \right)$ &  \\
\midrule
1 & 0 & $M$ \\
2 & 0.6$M$ & 0.4$M$ \\
3 & 0.8$M$ & 0.2$M$ \\
4 & $M-1$ & 1 \\
\bottomrule
\end{tabular}
\end{sc}
\end{small}
\end{center}
\end{table}

To highlight the heterogeneous mini-batch setting for the workers, we introduce the terminology called \textit{batch mode} outlined in Table \ref{tab:batchmode}. The workers' average batch size is 64, with a fixed small batch size of 4. Each batch mode incorporates a distinct distribution of workers with varying large and small batch sizes. For instance, Batch mode 1 represents a homogeneous mini-batch setting. Conversely, batch mode 4 indicates a scenario where all workers, excluding worker 1, use batch size 4 to compute the stochastic gradient per iteration. As the mode index increases, the number of workers with large batch sizes decreases, with dominance shifting towards those with smaller batch sizes.

\subsection{Effects of Heterogeneous Batch Sizes $B_m$}

\begin{figure}[!t]
    \centerline{
    \subfloat[CIFAR-10 dataset]{%
        \includegraphics[width=0.5\columnwidth]{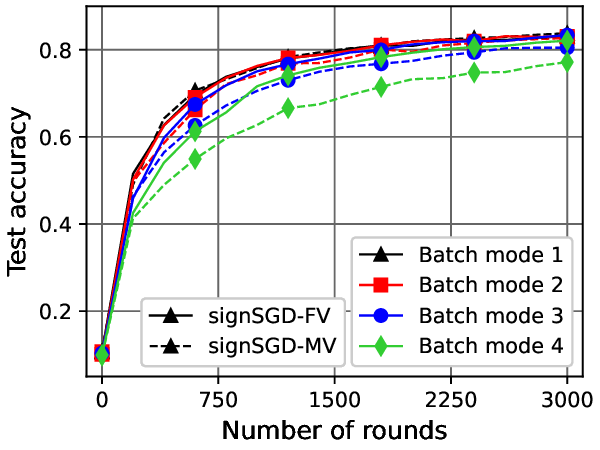}}
    \hfill
    \subfloat[MNIST dataset]{%
        \includegraphics[width=0.5\columnwidth]{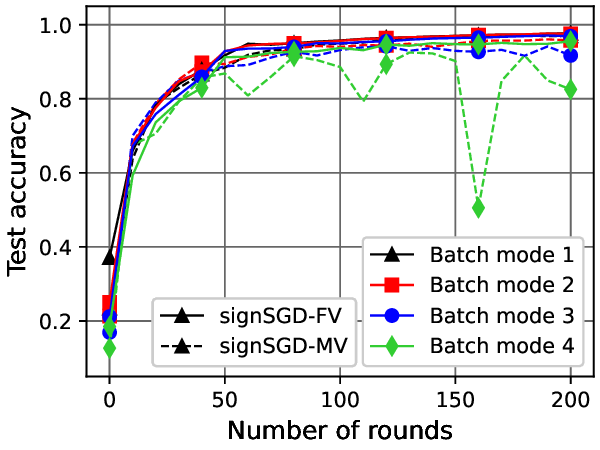}}
    }
    \caption{Test accuracy vs. training rounds varying the batch mode with $M = 15$ and $T_\mathsf{in} = 100$.}
    \label{fig:results_batch} 
\end{figure}

In Fig. \ref{fig:results_batch}, test accuracy is depicted versus training rounds, with $M=15$ and $T_\mathsf{in} = 100$. In a homogeneous batch settings (batch mode 1), signSGD-MV and signSGD-FV exhibit similar test accuracies across datasets. However, as the mini-batch mode index increases, signSGD-FV consistently outperforms signSGD-MV. The performance gap peaks in batch mode 4, even leading to non-convergence of signSGD-MV in the case of MNIST. This highlights the instability of signSGD-MV in heterogeneous batch settings. The observed improvement of signSGD-FV, especially in the presence of dominant workers in distributed learning systems, suggests that LLR weights effectively distinguish workers' computing capabilities. This aligns with insights from Sec. \ref{sec:error_bound}, emphasizing the efficacy of LLR weights in addressing challenges posed by disparate worker contributions during training.

\subsection{Effects of the Number of Workers $M$}

\begin{figure}[!t]
    \centerline{
    \subfloat[CIFAR-10 dataset]{%
        \includegraphics[width=0.5\columnwidth]{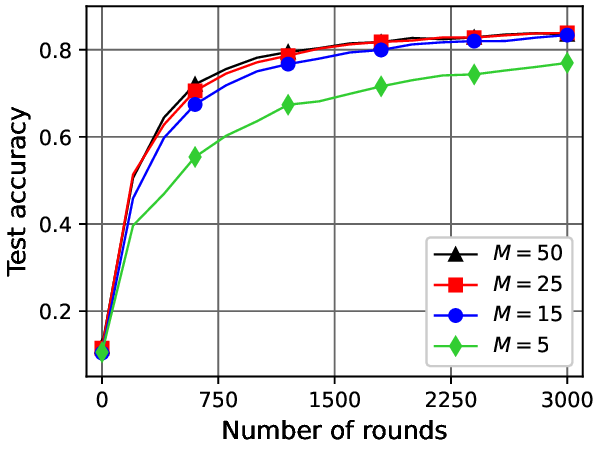}}
    \hfill
    \subfloat[CIFAR-100 dataset]{%
        \includegraphics[width=0.5\columnwidth]{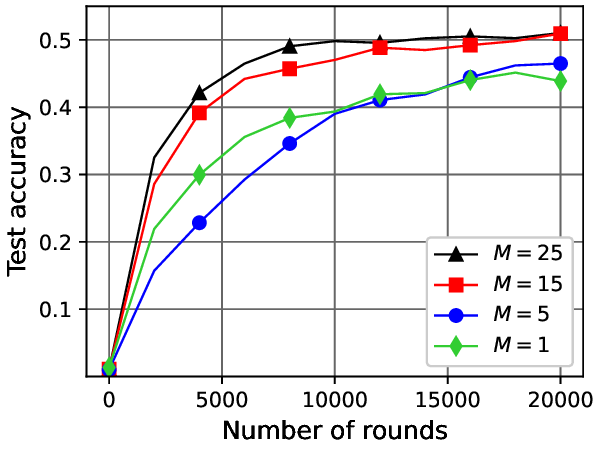}}
    }
    \caption{Test accuracy vs. training rounds on signSGD-FV varying the number of workers where the batch mode is 3 and $T_\mathsf{in} = 100$.}
    \label{fig:results_numworkers} 
\end{figure}

Fig. \ref{fig:results_numworkers} shows the test accuracy varying the number of workers $M$, with a fixed batch mode of 3 and $T_\mathsf{in} = 100$. From the CIFAR-10 and CIFAR-100 datasets, a clear trend can be seen in which the test accuracy of signSGD-FV consistently improves as the number of workers $M$ increases. The impact of augmenting $M$ becomes more pronounced when $M$ is small. This is attributed to the fact that the decrease in the exponential function, which is the form of decoding error bound as \eqref{eqn:WMV_errorbound} in Sec. \ref{sec:error_bound}, is much more significant in small $M$ region. Despite the above tendency, according to the results of $M=1$ and $M=5$ cases in Fig. \ref{fig:results_numworkers}, the convergence rate of signSGD-FV might be slower if the number of workers is insufficient. This aligns with the analysis in Sec. \ref{sec:im} that the uncertainty of estimated LLR weights becomes larger as $M$ gets smaller.


\subsection{Effects of Weight Uncertainty} \label{sec:result_init}

\begin{figure}[!t]
    \centerline{
    \subfloat[$\hat{p}_{m,n}^t$ with additive noise]{%
        \includegraphics[width=0.5\columnwidth]{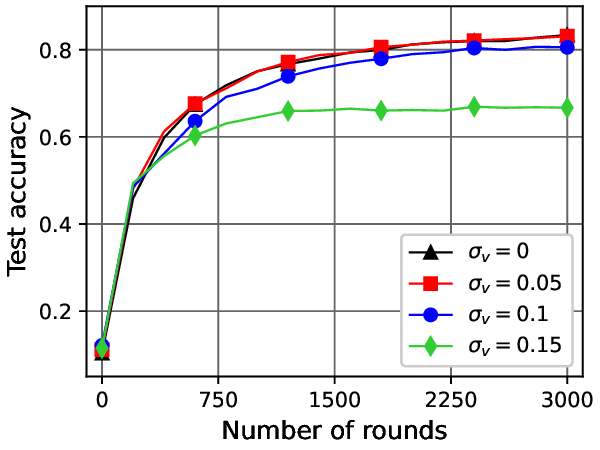}}
    \hfill
    \subfloat[Varying initial phase duration]{%
        \includegraphics[width=0.5\columnwidth]{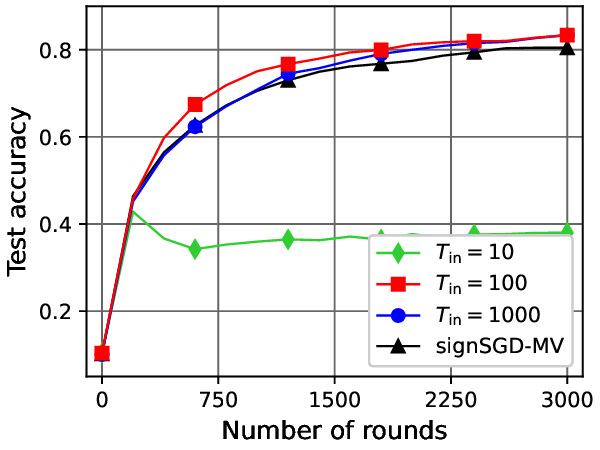}}
    }
    \caption{Test accuracy comparison on CIFAR-10 dataset by varying the uncertainty of estimated LLR weights for the batch mode 3 and $M = 15$.}
    \label{fig:results_initial} z
\end{figure}

In Fig. \ref{fig:results_initial}-(a), the test accuracy outcomes are presented, depicting the impact of Gaussian noise $v_{m,n}^t \sim \mathcal{N} \left( 0, \sigma_v^2 \right)$ for each $m \in [M]$, $n \in [N]$, and $t \in [T]$ on the estimated computing error probabilities, i.e., $\hat{p}_{m,n}^t + v_{m,n}^t$. The results illustrate a degradation in test accuracy as the noise variance (with values of $0.05$, $0.1$, and $0.15$) increases, aligning with the findings in Sec. \ref{sec:im}. This underscores the need for careful design in estimating $\hat{p}_{m,n}^t$ for signSGD-FV to fully harness the benefits of LLR weights. In Fig. \ref{fig:results_initial}-(b), the variation in test accuracy is explored concerning the initial MV decoding duration in Algorithm \ref{alg:signSGD-FV}. Notably, the accuracy experiences a decrease when the duration deviates from $T_\mathsf{in}=100$. This observation emphasizes that both a shortage of decoding error samples and an excessive bias towards the MV decoder can detrimentally affect the learning performance of signSGD-FV.

\subsection{Test Accuracy with Communication-Efficiency} \label{sec:alg_comp}

\begin{table}[t] 
\caption{Distributed learning algorithms}
\vspace{-1em}
\label{tab:benchmark}
\begin{center}
\begin{small}
\begin{sc}
\begin{tabular}{P{0.3\columnwidth} P{0.6\columnwidth}}
\toprule
Algorithms & Total communication costs \\
\midrule
SGD \cite{zinkevich2010parallelized} & $[32MN + 32MN] \times T$ \\
Top-K SGD \cite{stich2018sparsified} & $\left[ M \! \left( 32K + K \! \log_2 \! \left( \frac{N}{K} \right) \right) + 32MN \right] \times T$ \\
signSGD-MV \cite{bernstein2018asignsgd} & $[MN + MN] \times T$ \\
signSGD-FV & $[MN + MN] \times T$ \\
\bottomrule
\end{tabular}
\end{sc}
\end{small}
\end{center}
\end{table}

\begin{figure}[!t]
\centerline{
\subfloat[CIFAR-10 dataset with $M=25$]{%
    \includegraphics[width=0.5\columnwidth]{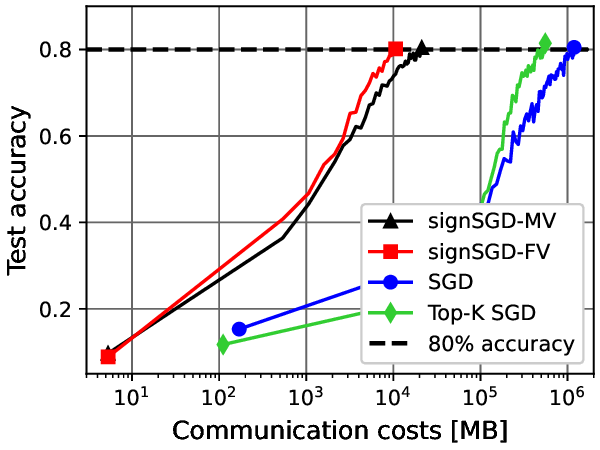}}
\hfill
\subfloat[MNIST dataset with $M=15$]{%
    \includegraphics[width=0.5\columnwidth]{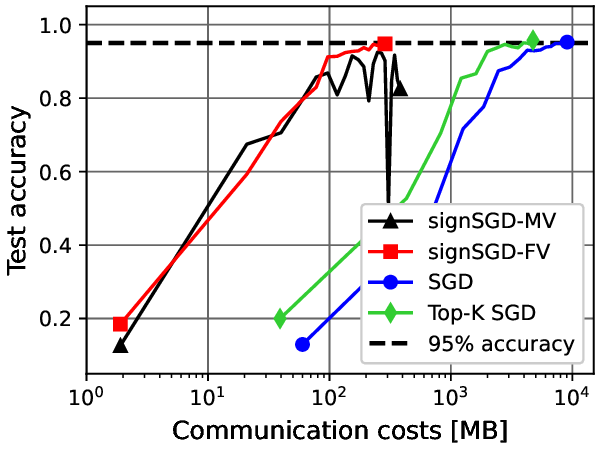}}
}
\caption{Test accuracy vs. total communication cost by comparing signSGD-FV with other conventional learning optimizers where $T_\mathsf{in}=100$, and the batch mode is 4.}
\label{fig:results_costcomp} 
\end{figure}

Fig. \ref{fig:results_costcomp} shows the learning performance in terms of communication costs with parameters set to $T_\mathsf{in}=100$ and batch mode 4. Table \ref{tab:benchmark} shows the benchmark algorithms and their associated costs. Top-$K$ SGD exhibits a sparsity level of $\frac{K}{N} = 0.1$. Notably, signSGD-FV, utilizing only 1 bit for each gradient coordinate, achieves a remarkable 30x reduction in communication costs for the MNIST dataset and an impressive 100x reduction for the CIFAR-10 dataset compared to distributed SGD optimizers. Despite Top-$K$ SGD offering a 2x reduction in communication costs compared to distributed SGD, it still incurs significantly higher costs than signSGD-FV. Additionally, while signSGD-MV can notably reduce communication costs compared to other conventional optimizers, its achieved test accuracy is considerably lower than signSGD-FV, as evidenced by the divergence of the model before reaching 95\% accuracy in Fig. \ref{fig:results_costcomp}-(b).

\section{Conclusion}

In this paper, we have introduced a communication-efficient distributed learning algorithm, signSGD-FV, based on federated voting. The key concept of federated voting is to exploit learnable weights when performing weighted majority voting for aggregating the one-bit gradient vectors per iteration. These weights play a crucial role in decoding the sign of aggregated local gradients, minimizing sign decoding errors when the edge devices hold heterogeneous computing capabilities. We also have provided a convergence rate analysis framework applicable to scenarios where the estimated weights are either perfectly or imperfectly known to the parameter server. Our analysis demonstrates that the proposed signSGD-FV guarantees convergence even in cases where edge devices employ heterogeneous mini-batch sizes. Simulation results show that the proposed algorithm outperforms the existing distributed learning algorithms in terms of convergence rates. 

An intriguing future research direction involves analyzing the convergence rates of the proposed algorithm in various practical scenarios, encompassing challenges such as adversarial attacks and diverse data distributions across heterogeneous edge devices.

{\appendices

\section{Proof of Theorem \ref{thm:1}} \label{sec:proof_thm1}
\begin{proof}
   From Assumption \ref{ass:2}, the difference between $f^{t+1}$ and $f^{t}$ is upper bounded by
    \begin{align}
        f^{t+1} - f^t &\le \left\langle \bar{\mathbf{g}}^t, \mathbf{x}^{t+1} - \mathbf{x}^t \right\rangle + \sum_{n=1}^N \frac{L_n}{2} \left( x_n^{t+1} - x_n^t \right)^2 \nonumber \\
        &= \sum_{n=1}^N \left[ - \bar{g}_n^t \cdot \delta \hat{U}_n^t + \frac{L_n}{2} \left( - \delta \hat{U}_n^t \right)^2 \right] \nonumber \\
        &= - \delta \sum_{n=1}^N \bar{g}_n^t \hat{U}_n^t + \frac{1}{2} \delta^2 \lVert \mathbf{L} \rVert_1 \nonumber \\
        &= - \delta \lVert \bar{\mathbf{g}}^t \rVert_1 + \frac{\delta^2}{2} \lVert \mathbf{L} \rVert_1 + 2\delta \sum_{n=1}^N \left| \bar{g}_n^t \right| \mathbf{1}_{ \left\{ U_n^t \ne \hat{U}_n^t \right\} }. \label{eqn:rate_MV1}
    \end{align}
    By taking expectation according to the randomness of $\hat{U}_n^t$, $f^{t+1} - f^t$ conditioned on $\mathbf{x}^t$ can be upper bounded by
    \begin{align}
        & \hspace{-1em} \mathbb{E} \left[ \left. f^{t+1} - f^t \right| \mathbf{x}^t \right] \nonumber \\
        & \le - \delta \lVert \bar{\mathbf{g}}^t \rVert_1 + \frac{\delta^2}{2} \lVert \mathbf{L} \rVert_1 + 2 \delta \sum_{n=1}^N \left| \bar{g}_n^t \right| \mathbb{P} \left[ U_n^t \ne \hat{U}_n^t \right] \nonumber \\
        & \le - \delta \lVert \bar{\mathbf{g}}^t \rVert_1 + \frac{\delta^2}{2} \lVert \mathbf{L} \rVert_1 + 2 \delta P_{\sf E}^{\sf max} \sum_{n=1}^N \left| \bar{g}_n^t \right| \nonumber \\
        &= - \delta \left( 1 - 2 P_{\sf E}^{\sf max}
         \right) \lVert \bar{\mathbf{g}}^t \rVert_1 + \frac{\delta^2}{2} \lVert \mathbf{L} \rVert_1. \label{eqn:rate_MV2}
    \end{align}
    Next, we take the expectation over $\mathbf{x}^t$, and apply a telescoping sum over the iterations, which provides
    \begin{align}
        f^1 - f^\star & \ge f^1 - \mathbb{E} \left[ f^T \right] \nonumber \\
        &= \mathbb{E} \left[ \sum_{t=1}^{T} f^t - f^{t+1} \right] \nonumber \\
        & \ge \mathbb{E} \left[ \sum_{t=1}^{T} \left\{ \delta \left( 1 - 2 P_{\sf E}^{\sf max} \right) \lVert \bar{\mathbf{g}}^t \rVert_1 - \frac{\delta^2}{2} \lVert \mathbf{L} \rVert_1 \right\} \right], \nonumber \\
        &= \delta \left( 1 - 2 P_{\sf E}^{\sf max} \right) \mathbb{E} \left[ \sum_{t=1}^{T} \lVert \bar{\mathbf{g}}^t \rVert_1 \right] -\frac{\delta^2 T}{2} \lVert \mathbf{L} \rVert_1, \label{eqn:rate_MV7}
    \end{align}
    where the last equality holds when $\delta$ is fixed according to the training round $t \in [T]$. Consequently, by substituting the learning rate $\delta = \sqrt{\frac{2 \left( f^1 - f^\star \right)}{T \lVert \mathbf{L} \rVert_1}}$ into \eqref{eqn:rate_MV7}, we obtain 
    \begin{align}
        \mathbb{E} \left[ \frac{1}{T} \sum_{t=1}^{T} \lVert \bar{\mathbf{g}}^t \rVert_1 \right] & \le \frac{1}{1 - 2 P_{\sf E}^{\sf max}} \left[ \frac{1}{\delta T} \left( f^1 - f^\star \right) + \frac{\delta}{2} \lVert \mathbf{L} \rVert_1 \right] \nonumber \\
        &= \frac{1}{1 - 2 P_{\sf E}^{\sf max}} \sqrt{\frac{2 \left( f^1 - f^\star \right) \lVert \mathbf{L} \rVert_1}{T}}. \label{eqn:rate_MV8}
    \end{align}
    This completes the proof. 
\end{proof}

\section{Proof of Theorem \ref{thm:2}} \label{sec:proof_thm2}

\begin{proof}
We commence by defining a binary random variable that indicates the decoding error event, i.e.,  $Z_{m,n}^t = {\bf 1}_{\left\{ U_n^t \ne Y_{m,n}^t \right\}}$. When employing WMV decoding with an LLR weight $w_{m,n}^t=\ln \frac{1 - p_{m,n}^t}{p_{m,n}^t}$, decoding failures arise if the cumulative sum of weights assigned to the incorrectly decoding workers exceeds half of the total weight. Using this fact, the decoding error probability is rewritten as
\begin{align}
    &\mathbb{P} \left[ U_n^t \ne {\hat U}_{{\sf WMV}, n}^t \right] \nonumber \\
    &= \mathbb{P} \left[ \sum_{m=1}^M w_{m,n}^t Z_{m,n}^t \ge \frac{1}{2} \sum_{m=1}^M w_{m,n}^t \right] \nonumber \\
    &= \mathbb{P} \left[ \sum_{m=1}^M w_{m,n}^t \left( Z_{m,n}^t - p_{m,n}^t \right) \ge \sum_{m=1}^M w_{m,n}^t \left( \frac{1}{2} - p_{m,n}^t \right) \right] \nonumber \\
    &= \mathbb{P} \left[ \sum_{m=1}^M w_{m,n}^t \bar{Z}_{m,n}^t \ge \eta \right], \label{eqn:prob_WMV1}
\end{align}
where $\bar{Z}_{m,n}^t = Z_{m,n}^t - p_{m,n}^t$ and $\eta = \sum_{m=1}^M w_{m,n}^t \left( \frac{1}{2} - p_{m,n}^t \right)$. Applying the Chernoff bound to \eqref{eqn:prob_WMV1} for $s>0$ yields an upper bound on the error probability, expressed as
\begin{align}
    \mathbb{P} \left[ \sum_{m=1}^M \! w_{m,n}^t \bar{Z}_{m,n}^t \!\! \ge \eta \right] & \! \le \min_{s>0} e^{-\eta s} \, \mathbb{E} \! \left[ \exp \! \left( \! s \! \sum_{m=1}^M \! w_{m,n}^t \bar{Z}_{m,n}^t \! \right) \! \right] \nonumber \\
    & \! = \min_{s>0} e^{-\eta s} \prod_{m=1}^M \mathbb{E} \left[ e^{s w_{m,n}^t \bar{Z}_{m,n}^t} \right].\label{eqn:prob_WMV2}
\end{align}
Leveraging the large deviation bound established in Lemma \ref{lem:1}, we obtain an upper bound on the expectation term in \eqref{eqn:prob_WMV2} as    
\begin{align}
    \mathbb{E} \! \left[ e^{s w_{m,n}^t \bar{Z}_{m,n}^t} \right] \! &= p_{m,n}^t e^{s w_{m,n}^t \left( 1 - p_{m,n}^t \right)} \!+\! \left( 1 \!-\! p_{m,n}^t \right) e^{-s w_{m,n}^t p_{m,n}^t} \nonumber \\
    & \le \exp \left[ \frac{1 - 2 p_{m,n}^t}{4 \ln \frac{1 - p_{m,n}^t}{p_{m,n}^t}} \left( w_{m,n}^t \right)^2 s^2 \right] \nonumber \\
    &= \exp \left[ \frac{1}{2} \left( \frac{1}{2} - p_{m,n}^t \right) w_{m,n}^t s^2 \right], \label{eqn:prob_WMV3}
\end{align}
where the last equality follows from $w_{m,n}^t=\ln \frac{1 - p_{m,n}^t}{p_{m,n}^t}$.
Invoking \eqref{eqn:prob_WMV3} into \eqref{eqn:prob_WMV2}, and also using $\eta = \sum_{m=1}^M w_{m,n}^t \left( \frac{1}{2} - p_{m,n}^t \right)$, the upper bound on the WMV decoding error probability becomes
\begin{align}
    \mathbb{P} \left[ U_n^t \! \ne \! {\hat U}_{{\sf WMV}, n}^t \right] 
    \! & \le \min_{s>0} e^{-\eta s} \exp \left[ \frac{s^2}{2} \! \sum_{m=1}^M \! \left( \frac{1}{2} \! - \! p_{m,n}^t \! \right) \! w_{m,n}^t \right] \nonumber \\
    & =  \min_{s>0} \exp \left( \frac{1}{2} \eta s^2 - \eta s \right) \nonumber \\
    & = \exp \left( - \frac{1}{2} \eta \right), \label{eqn:proof_thm4_last3}
\end{align}
where the last equality follows from the fact that $s = 1$ is the minimizer of the optimization problem in \eqref{eqn:proof_thm4_last3}. Then, the upper bound becomes
\begin{align}
    \mathbb{P} \left[ U_n^t \ne \hat{U}_{\mathsf{WMV}, n}^t \right] & \le \exp \left(- M \gamma^{\sf WMV}  \right), \label{eqn:proof_thm4_last2}
\end{align}
with 
\begin{align}
   \gamma^{\sf WMV} = \frac{1}{2M} \sum_{m=1}^M \left( \frac{1}{2} - p_{m,n}^t \right) \ln \frac{1 - p_{m,n}^t}{p_{m,n}^t}.
\end{align}
\end{proof}

\section{Proof of Lemma \ref{lem:1}} \label{sec:proof_lem1}
\begin{proof}
    The computation error probability $p_{m,n}^t$ is upper bounded by
    \begin{align}
        p_{m,n}^t &= \mathbb{P} \left[ Y_{m,n}^t \ne U_n^t \right] \nonumber \\
        & \le \mathbb{P} \left[ \left| g_{m,n}^t - \bar{g}_n^t \right| \ge \left| \bar{g}_n^t \right| \right] \nonumber \\
        &\mathop {\le}^{(a)} \frac{\mathbb{E} \left[ \left| g_{m,n}^t - \bar{g}_n^t \right| \right]}{\left| \bar{g}_n^t \right|} \nonumber \\
        & \mathop {\le}^{(b)} \frac{\sqrt{\mathbb{E} \left[ \left| g_{m,n}^t - \bar{g}_n^t \right|^2 \right]}}{\left| \bar{g}_n^t \right|} \nonumber \\
        & = \frac{\sigma}{\sqrt{B_m}},
    \end{align}
    where (a) and (b) follows from  Markov's and Jensen's inequalities. The last equality is due to Assumption \ref{ass:3}.
\end{proof}

\section{Proof of Corollary \ref{cor:2}} \label{sec:proof_cor2}
\begin{proof}
    Let us begin the proof with the decoding error bound of WMV in \eqref{eqn:WMV_errorbound}. If $0 < p_{m,n}^t < \frac{1}{4}$, we can obtain the upper bound as
    \begin{align} \label{eqn:cor2_1}
        P_\mathsf{E}^\mathsf{WMV} &\le \exp \left[ -\frac{1}{2} \sum_{m=1}^M \ln \frac{1 - p_{m,n}^t}{p_{m,n}^t} \left( \frac{1}{2} - p_{m,n}^t \right) \right] \nonumber \\
        &= \prod_{m=1}^M \left( \frac{p_{m,n}^t}{1 - p_{m,n}^t} \right)^{\frac{1}{2} \left( \frac{1}{2} - p_{m,n}^t \right)} \nonumber \\
        &\le \prod_{m=1}^M \left( \frac{4}{3} p_{m,n}^t \right)^{\frac{1}{2} \left( \frac{1}{2} - p_{m,n}^t \right)} \nonumber \\
        &\le \prod_{m=1}^M \left( \frac{4}{3} p_{m,n}^t \right)^\frac{1}{8}.
    \end{align}
    Leveraging the computation error bound in Lemma \ref{lem:1}, the upper bound in \eqref{eqn:cor2_1} becomes
    \begin{align}
        P_\mathsf{E}^\mathsf{WMV} &\le \prod_{m=1}^M \left( \frac{4}{3} \frac{\sigma}{\sqrt{B_m}} \right)^\frac{1}{8} \nonumber \\
        &= \left[ \frac{4}{3} \sigma \cdot \left( \prod_{m=1}^M \frac{1}{\sqrt{B_m}} \right)^\frac{1}{M} \right]^\frac{M}{8} \nonumber \\
        &= \exp \left[ -\frac{M}{8} \ln \left( \frac{3}{4} \frac{\sqrt{B_\mathsf{GM}}}{\sigma} \right) \right],
    \end{align}
    where $B_\mathsf{GM} = \left( \prod_{m=1}^M B_m \right)^\frac{1}{M}$ is the geometric mean of workers' batch sizes.
\end{proof}

\section{Proof of Theorem \ref{thm:3}} \label{sec:thm3}
\begin{proof}
Building upon the findings obtained in \eqref{eqn:prob_WMV2}, we commence our analysis by considering the upper bound of the error probability with imperfect knowledge of $\hat{w}_{m,n}^t$ as
\begin{align}
        \mathbb{P} \left[ U_n^t \ne \hat{U}_{{\sf WMV}, n}^t \right] &\le \mathbb{P} \left[ \sum_{m=1}^M \hat{w}_{m,n}^t \bar{Z}_{m,n}^t \ge \hat{\eta} \right], \nonumber \\
        &\le \min_{s>0} e^{-\hat{\eta} s} \prod_{m=1}^M \mathbb{E} \left[ e^{s \hat{w}_{m,n}^t \bar{Z}_{m,n}^t} \right], \label{eqn:im1}
    \end{align}
    where $\hat{\eta} = \sum_{m=1}^M \hat{w}_{m,n}^t \left( \frac{1}{2} - p_{m,n}^t\right)$. From Lemma \ref{lem:1}, we compute the upper bound of the expectation term in \eqref{eqn:im1} as
    \begin{align}
        \mathbb{E} \! \left[ e^{s \hat{w}_{m,n}^t \bar{Z}_{m,n}^t} \right] \!
        & \le \exp \left[ \frac{1 - 2 p_{m,n}^t}{4 w_{m,n}^t} \left( \hat{w}_{m,n}^t \right)^2 s^2 \right] \nonumber \\
        & = \exp \left[ \frac{1}{2} \left( \frac{1}{2} - p_{m,n}^t \right) \frac{\hat{w}_{m,n}^t}{w_{m,n}^t} \hat{w}_{m,n}^t s^2 \right] \nonumber \\
        &\mathop {\le}^{(a)} \exp \left[ \frac{1 + \delta_{\sf max}}{2} \! \left( \frac{1}{2} \!-\! p_{m,n}^t \! \right) \! \hat{w}_{m,n}^t s^2 \right], \label{eqn:im2}
    \end{align}
where (a) is due to Assumption \ref{ass:4}. Substituting \eqref{eqn:im2} into \eqref{eqn:im1}, we obtain
    \begin{align}
        & \mathbb{P} \left[ U_n^t \ne \hat{U}_{{\sf WMV}, n}^t \right] \nonumber \\
        & \le \min_{s>0} e^{-\hat{\eta} s} \exp \left[ \frac{(1 + \delta_{\sf max}) s^2}{2} \sum_{m=1}^M \left( \frac{1}{2} - p_{m,n}^t \right) \hat{w}_{m,n}^t \right] \nonumber \\
        & =  \min_{s>0} \exp \left( \frac{(1 + \delta_{\sf max}) \hat{\eta}}{2}  s^2 - \hat{\eta} s \right) \nonumber \\
        & = \exp \left( - \frac{1}{2 (1 + \delta_{\sf max})} \hat{\eta} \right), \label{eqn:im3}
    \end{align}
    where the last equality holds when $s = \frac{1}{1 + \delta_{\sf max}}$. Now, we derive a lower bound of $\hat{\eta}$ in terms of $\delta_{\rm min}$ as
    \begin{align}
      \hat{\eta} &=   \sum_{m=1}^M \hat{w}_{m,n}^t \left( \frac{1}{2} - p_{m,n}^t\right) \nonumber\\
      &= \sum_{m=1}^M \left( \frac{1}{2} - p_{m,n}^t \right) \frac{\hat{w}_{m,n}^t}{w_{m,n}^t} \cdot w_{m,n}^t \nonumber\\
      &\ge (1 - \delta_{\sf min}) \sum_{m=1}^M \left( \frac{1}{2} - p_{m,n}^t \right) w_{m,n}^t. \label{eqn:eta}
    \end{align}
   Invoking \eqref{eqn:eta} into \eqref{eqn:im3}, we obtain
    \begin{align}
        \mathbb{P} \left[ U_n^t \ne \hat{U}_{{\sf WMV}, n}^t \right] 
        & \le \exp \left[ - M \left(\frac{1 - \delta_{\sf min}}{1 + \delta_{\sf max}}  \right) \gamma^{\sf WMV} \right]. \label{eqn:im4}
    \end{align}
   This completes the proof. 
\end{proof}

}

\bibliographystyle{IEEEtran}
\bibliography{bibfile}

\vfill

\end{document}